\documentclass{article}

\usepackage[numbers,square,comma]{natbib}

\usepackage[utf8]{inputenc} 
\usepackage[T1]{fontenc}    
\usepackage{hyperref}       
\usepackage{url}            
\usepackage{booktabs}       
\usepackage{amsfonts}       
\usepackage{nicefrac}       
\usepackage{microtype}      

\usepackage{amsmath, amssymb}
\usepackage{graphicx}
\usepackage{multirow}

\newcommand{\norm}[1]{\left\lVert#1\right\rVert}
\newcommand{\st}{\mathfrak{s}}
\newcommand{\ns}{\mathfrak{n}}
\newcommand{\spd}[1]{\mathcal{S}^+_{#1}}

\newcommand{\Gr}{\mathcal{G}}

\DeclareMathOperator{\tr}{tr} 
\DeclareMathOperator{\spn}{span}

\newtheorem{theorem}{Theorem}[section]
\newtheorem{definition}[theorem]{Definition}

\newtheorem{proposition}[theorem]{Proposition}
\newtheorem{corollary}[theorem]{Corollary}
\newenvironment{proof}{\emph{Proof.}}{\hfill \rule{2mm}{2mm}}

\title{Geometry-aware Stationary Subspace Analysis}

\usepackage{authblk}
\author[1]{Inbal Horev}
\author[2]{Florian Yger}
\author[1]{Masashi Sugiyama}
\affil[1]{Graduate School of Frontier Sciences, University of Tokyo}
\affil[2]{LAMSADE Universit{\'e} Paris-Dauphine}
\date{}

\begin{document}

\maketitle

\begin{abstract}
In many real-world applications observed data exhibits \emph{non-stationarity}, i.e., its distribution changes over time.
One approach to handling non-stationarity is to remove or minimize it before attempting to analyze the data.
In the context of brain computer interface (BCI) data analysis this may be done by means of \emph{stationary subspace analysis} (SSA).
The SSA method finds a matrix that projects the data onto a stationary subspace by optimizing a cost function based on a matrix divergence.
In this work we present an alternative method for SSA based on a symmetrized version of this matrix divergence.
We show that doing so frames the problem in terms of distances between symmetric positive definite (SPD) matrices, suggesting a geometric interpretation of the problem.
Stemming from this geometric viewpoint, we introduce and analyze a method which utilizes the geometry of the SPD matrix manifold and the invariance properties of its metrics.
We demonstrate the usefulness of our method in experiments on both synthesized and real-world data.
\end{abstract}

\section{Introduction}

A common assumption in statistical modeling is that the distribution of observed data does not change over time, i.e., that it is \emph{stationary}.
In most cases it is this assumption of stationarity which allows results to be effectively generalized from the sample to the population.
When the stationarity assumption is violated, as is often the case in real-world applications such as speech enhancement \citep{cohen2001speech} 
or neurological data analysis \citep{samek2012brain}, 
specialized machine learning methods must be developed in order to maintain adequate prediction capabilities. 

A relatively well-studied non-stationary setting is covariate-shift \citep{shimodaira2000improving}, in which the input distribution changes but the conditional distribution of the outputs does not.
The problem of covariate-shift has received growing attention in recent years, and many theoretical and practical aspects have been addressed (see \cite{sugiyama2012machine} for an in-depth exploration of this topic).

These works typically do not aim to remove or reduce non-stationarity in the data, but rather they try to cope with its existence.
A different approach is to remove, or minimize, any existing non-stationarities before attempting to analyze the collected data.
In the context of brain computer interface (BCI) data analysis, 
two such note-worthy methods are stationary subspace analysis (SSA) \cite{von2009finding} and stationary common spacial patterns (sCSP) \cite{wojcikiewicz2011stationary}.

Similar in spirit to independent component analysis (ICA) \cite{hyvarinen2004independent}, SSA statistically models the data as a mixture of stationary and non-stationary signals.
Unlike ICA, however, the signals are not assumed to be independent of each other.
The data is first split into (possibly overlapping) time frames called epochs.
Then a projection matrix is found by optimizing a cost function based on the divergence between distributions in various epochs.

The other method, by now quite a standard step for classification tasks in BCI systems, is the (supervised) sCSP method.
Its goal is to project the data onto a subspace in which the various data classes are more separable.
The sCSP method directs this subspace towards a stationary subspace by means of regularization.

Although SSA is essentially an unsupervised method, variations of it exist which are useful for supervised tasks such as classification \cite{samek2012brain}.
These methods attempt to remove non-stationarity while keeping the discriminative inter-class variations intact.
In this work we present an alternative method for stationary subspace extraction, focusing for the moment on the unsupervised setting.

Unlike SSA, the inputs to our method are not the raw signals themselves, but rather covariance matrices, computed from the signals (or from features based on the signals) using one of the many existing covariance estimators (e.g.,
\cite{ledoit2004well}). 
Covariance matrices have gained increasing attention in recent years, and are now commonly used in many machine learning and signal processing applications such as computer vision applications \cite{tuzel2006region}, 
brain imaging \cite{pennec2006riemannian} 
and BCI data analysis \cite{barachant2013classification}. 
Their rich mathematical structure has been extensively studied \cite{bhatia2009positive}, 
and advances in optimization methods on matrix manifolds in recent years have motivated the development of geometric methods for various machine learning tasks such as dictionary learning \cite{cherian2014riemannian}, 
metric learning \cite{kusner2014stochastic} 
and dimensionality reduction \cite{fletcher2004principal}.
This in turn motivates a covariance-based approach which utilizes the geometric properties of the symmetric positive definite (SPD) matrix manifold for SSA.

\section{Geometry-aware stationary subspace analysis}

As discussed in the introduction, the task of extracting the stationary part from an observed mixture of stationary and non-stationary signals is essential in various applications.
In this section we present our approach to this problem.
We name it \emph{geometry-aware SSA} (gaSSA) since we utilize the geometric properties of covariance matrices.
To find a stationary subspace, SSA uses a cost function which is based on a matrix divergence.
As a first step, we suggest a formulation that uses a symmetrized version of the same matrix divergence.
We then show that this symmetrized matrix divergence is 
a distance between SPD matrices and offer a geometric interpretation to the problem of SSA.

We begin with a formal statement of the problem.
To this end we provide a review of the original SSA model and framework \cite{von2009finding}.
Next, we present the symmetrized matrix divergence and discuss its relation to the Riemannian geometry of SPD matrices.
Finally, we end this section by introducing and analyzing a general formulation for our geometry-aware SSA.

\subsection{Stationary subspace analysis}
\label{section:SSA}

Let $x(t) \in \mathbb{R}^D$ be a vector of $D$ input signals, composed of $m$ stationary sources $s^\st(t) = \left[ s_1(t), \hdots, s_m(t)\right]^\top$ ($\st$-sources) and $D-m$ non-stationary sources $s^\ns(t) = \left[ s_{m+1}(t), \hdots, s_D(t)\right]^\top$ ($\ns$-sources), mixed by a linear mixing transformation,
\begin{equation}
x(t) = As(t) = \left[ \begin{matrix} A^\st  A^\ns \end{matrix} \right] \left[ \begin{matrix} s^\st (t) \\ s^\ns (t) \end{matrix} \right],
\end{equation}
where $A \in GL_D(\mathbb{R})$, the general linear group of size $D$ over $\mathbb{R}$, i.e., the set of all $D \times D$ invertible matrices with entries in $\mathbb{R}$.
The spaces spanned by the column vectors $A^\st$ and $A^\ns$ are referred to as the $\st$-space and $\ns$-space, respectively.

The SSA model makes relatively few assumptions on the $\st$- and $\ns$-sources.
First, the $\st$-sources are stationary only in the weak (or wide) sense \cite{kantz2004nonlinear}.
That is, their first and second moments are required to be constant in time.
For the $\ns$-sources, their first two moments may vary between epochs of length $T$ denoted by $\tau_i = \left[ t_0(i), \hdots, t_0(i) + T \right]$, where $t_0(i)$ is the start time of the $i$-th epoch.
The sources do not necessarily follow a Gaussian distribution, but non-stationarities are assumed to be visible in the first two moments.
Furthermore, this model does not assume that the sources are independent, namely, their covariance matrix is given by
\begin{equation}
\Sigma (\tau_i) = \mathbb{E}\left[ s(\tau_i) s(\tau_i)^\top \right] =  \left( \begin{matrix}
\Sigma^{\st} & \Sigma^{\st \ns}(\tau_i) \\
{\Sigma^{\st \ns}(\tau_i)}^\top & \Sigma^{\ns}(\tau_i)
\end{matrix} \right),
\end{equation}
where $\Sigma^{\st} \in \mathbb{R}^{m \times m}$, $\Sigma^{\ns} \in \mathbb{R}^{(D-m) \times (D-m)}$ and $\Sigma^{\st \ns} \in \mathbb{R}^{m \times (D-m)}$.
Note that since $\Sigma^{\st}$ is time independent we have dropped the notation $(\tau_i)$.

The goal of SSA is to find a de-mixing transformation $\widehat{A}^{-1}$ that separates the $\st$-sources from the $\ns$-sources.
This matrix $\widehat{A}^{-1}$ is not unique, but rather undetermined up to scaling, sign and linear transformations within the $\st$- and $\ns$-spaces.
So, without loss of generality, the data may be centered and whitened such that the $\st$-sources have a zero mean and a diagonal covariance matrix with unit variance.\footnote{This is also common practice in ICA \cite{hyvarinen2004independent}}
Put differently, the de-mixing matrix is written as $\widehat{A}^{-1} = \widehat{B} Z$ where $Z = \text{Cov}(x)^{-1/2}$ 
is a whitening matrix created by a covariance estimator $\text{Cov}(\cdot )$ (in this case it is the empirical estimator) and $\widehat{B} \in \mathcal{O}_D = \left\lbrace V \in \mathbb{R}^{D \times D} \; : \; V^\top V = \mathbb{I} \right\rbrace$, the set of all $D \times D$ orthogonal matrices.

To find the matrix $\widehat{B}$, the signals are split into $N$ epochs $\tau_1,\hdots, \tau_N$ of length $T$.
Each epoch is characterized by its empirical mean $\widehat{\mu}_i$ and covariance $\widehat{\Sigma}_i$.
Then for each epoch, the mean and covariance of the $\st$-sources may be written as
\begin{equation}
\label{eq:epochStats}
\widehat{\mu}^\st_i = \mathbb{I}_D^m \widehat{B}Z \widehat{\mu}_i, \qquad \widehat{\Sigma}^\st_i = \mathbb{I}_D^m \widehat{B}Z \widehat{\Sigma}_i \left( \mathbb{I}_D^m \widehat{B}Z  \right)^\top,
\end{equation}
where $\mathbb{I}_D^m$ is the $D \times D$ identity matrix, truncated to the first $m$ columns.
Since the true $\mu^\st$ and $\Sigma^\st$ are by definition stationary, the matrix $\widehat{B}$ is that for which $\widehat{\mu}^\st_i$ and $\widehat{\Sigma}^\st_i$ vary the least across all epochs.
Owing to the maximum entropy principle, 
SSA uses the Kullback-Leibler (KL) divergence between Gaussian distributions to compare the epoch distributions up to their second moment.
The matrix $\widehat{B}$ is thus found by minimizing the following cost function:
\begin{equation}
\label{eq:SSAcost}
\mathcal{L}\left({\widehat{B}}\right) = \sum_{i=1}^N {D_\mathrm{KL}\left[ \mathcal{N}\left( \widehat{\mu}^\st_i, \widehat{\Sigma}^\st_i \right) \; \vert \vert \; \mathcal{N} (0,\mathbb{I}) \right]} = - \sum_{i=1}^N \left( {\log \det \widehat{\Sigma}^\st_i + {\widehat{\mu}^\st_i}{}^\top \widehat{\mu}^\st_i} \right),
\end{equation}
where $D_\mathrm{KL}$ is the KL divergence and $\mathcal{N}\left( \mu, \Sigma \right)$ denotes a multivariate Gaussian distribution with mean $\mu$ and covariance $\Sigma$.


\subsection{Symmetrized matrix divergence}

SSA and its variants use in their cost function the KL divergence between Gaussian distributions.
In what follows we assume that these distributions have a zero mean.
Under this assumption the KL divergence is a Bregman matrix divergence \cite{banerjee2005clustering} between covariance matrices.
The family of Bregman matrix divergences is generally defined as
\begin{equation}
D_{\Phi}\left( X,Y\right) = \Phi (X) - \Phi (Y) - \tr \left( \left( \nabla \Phi \left( Y \right) \right)^\top \left( X-Y\right) \right).
\end{equation}
$D_\mathrm{KL}$ is obtained for  $\Phi (X) = - \log \det X$, so it is often called the log-determinant divergence \cite{kulis2009low}.

Bregman matrix divergences are useful in machine learning and have a number of useful properties \cite{kulis2009low}, 
such as linearity and convexity in the first argument (and, in the case of the KL divergence, also in the second).
However, as can be seen from their definition, they are asymmetric and do not satisfy the triangle equality.
In particular, in our case we have that $D_\mathrm{KL} \left( X \; \vert \vert \; Y \right) \neq D_\mathrm{KL} \left( Y \; \vert \vert \; X \right)$ for two arbitrary matrices $X \neq Y$.
Subsequently, symmetrized versions of the Bregman matrix divergence, in particular Jensen-Bregman divergences, have been studied in recent years \cite{nielsen2009sided}. 
In the case of the KL divergence (for zero mean distributions), this gives
\begin{eqnarray}
\nonumber D_\mathrm{JBLD} \left( X, Y \right) & = &\frac{1}{2} \left[ D_\mathrm{KL} \left( X \; \vert \vert \; \frac{1}{2}\left(X+Y\right) \right) + D_\mathrm{KL} \left( Y \; \vert \vert \; \frac{1}{2}\left(X+Y\right) \right) \right] \\
 & = & \log \det \left( \frac{1}{2}\left(X+Y\right)\right) - \frac{1}{2} \log \det \left( XY \right)
\end{eqnarray}
and is called the Jensen-Bregman log-determinant (JBLD) divergence \cite{cherian2011efficient}.

The JBLD has many favorable properties (see \cite{sra2011positive}), primarily that its square root comprises a metric on the SPD matrix manifold.
Moreover, in \citet{sra2011positive} it has been shown that it is a close approximation to the \emph{affine invariant Riemannian metric} (AIRM) \cite{bhatia2009positive} and shares many of its mathematical properties.
The practical properties of both metrics, in particular their invariance properties, will be discussed in Section \ref{section:invariance}.
In the context of SPD matrices, the JBLD is referred to as the symmetric Stein divergence or the (square of the) log-determinant metric \cite{sra2011positive}.
In the sequel we adopt the notation $\delta_\mathrm{s} \left( X , Y \right)$.

Motivated by the above, we formulate SSA using a new cost function based on $\delta_\mathrm{s}$ (cf. Eqs. (\ref{eq:epochStats}) and (\ref{eq:SSAcost})):
\begin{align}
\!\!\!
\mathcal{L}({\widehat{B}}) = \sum_{i=1}^N  {\delta_\mathrm{s}^2( \widehat{\Sigma}^\st_i, \mathbb{I} )} = \sum_{i=1}^N \delta_\mathrm{s}^2 ( Q^\top \tilde{\Sigma_i} Q , \mathbb{I})
 = \sum_{i=1}^N \left[ \log \det \left( \frac{1}{2} (\widehat{\Sigma}^\st_i +\mathbb{I} ) \right) - \frac{1}{2} \log \det ( \widehat{\Sigma}^\st_i ) \right],
\label{eq:costDeltaS}
\end{align}
where $Q = \left( \mathbb{I}_D^m \widehat{B} \right)^\top$ and $\tilde{\Sigma}_i = Z \Sigma_i Z^\top$ are the matrices whitened with $Z = {\bar{\Sigma}}^{-1/2}$, $\bar{\Sigma} = \underset{\Sigma \in \spd{D}}{\mathrm{argmin}} \; \delta_\mathrm{s}^2 \left( \Sigma_i, \Sigma \right)$, the mean of $\Sigma_i$ w.r.t.~$\delta_\mathrm{s}$.
$\spd{D}$ denotes the set of all $D \times D$ SPD matrices.

\subsection{A geometric interpretation}

By replacing the cost function with one based on the symmetrized divergence we gain not only the beneficial properties of the symmetric divergence, but also new insight into the problem of SSA.
First, note that in Eq. (\ref{eq:costDeltaS}) the problem is ultimately framed in terms of distances between SPD matrices.
This suggests adopting a geometric perspective, whereby the notion of stationarity is captured by the dispersion of the matrices $\Sigma_i$.
In this view, the assumption that the covariance matrices of stationary signals do not vary much between epochs translates to them having small distances between them.

\begin{figure}[t]
\begin{minipage}[b]{0.36\textwidth}
\centering
  \includegraphics[width=1\textwidth]{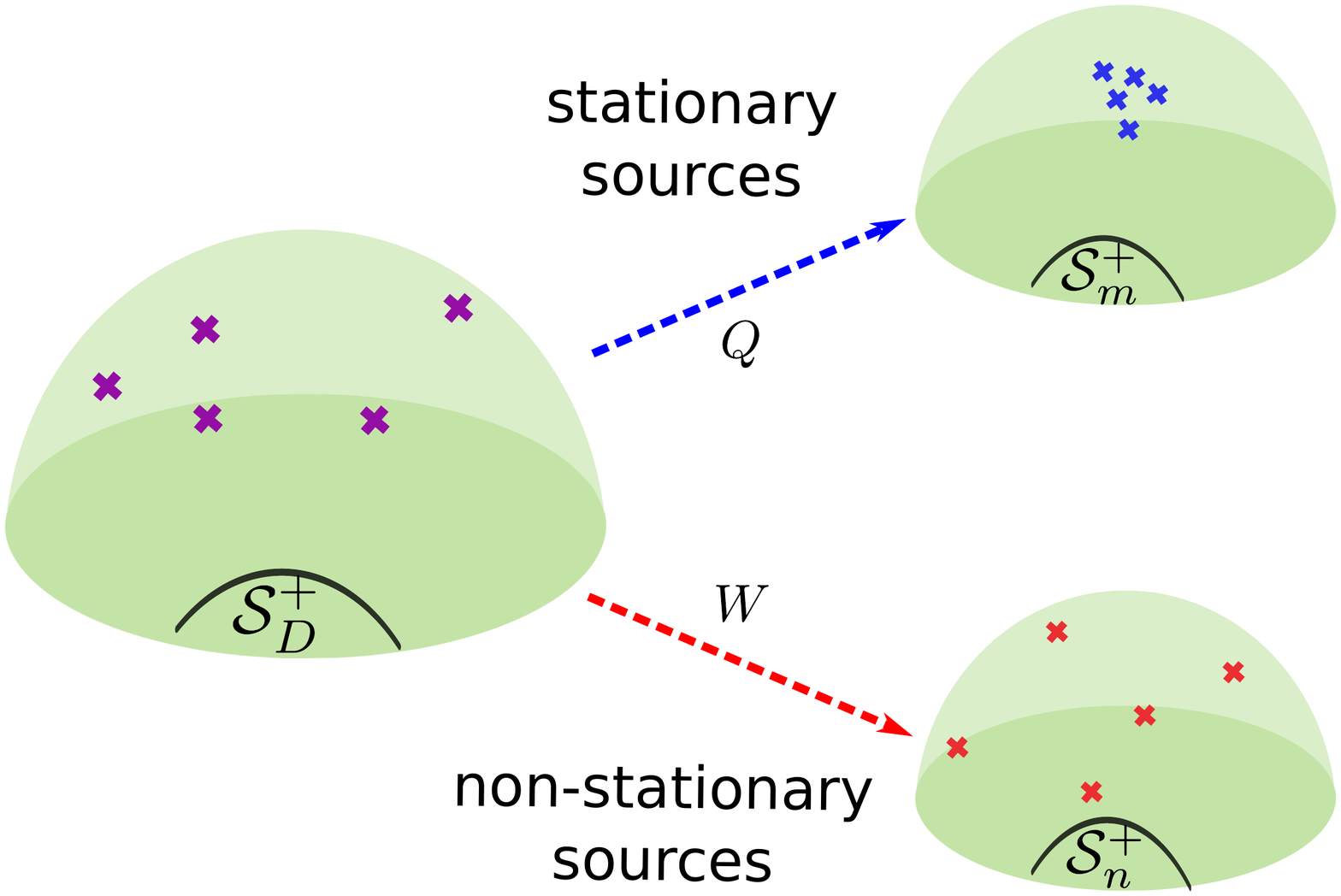}
\end{minipage}
~~~~~
\begin{minipage}[b]{0.62\textwidth}
\caption{An illustration of our covariance-based approach.
A set of input covariance matrices $\left\lbrace \Sigma_i \in \spd{D} \right\rbrace $ are made up of a mixture (purple) of stationary (blue) and non-stationary (red) parts.
Due to non-stationarities, in the original space the matrices are spread out.
The matrices are mapped to two lower dimensional spaces, namely the $\st-$ (stationary) and $\ns-$ (non-stationary) spaces, via the matrices $Q$ and $W$, respectively.
In the $\st$-space the matrices are now more localized compared to the $\ns$-space in which they have a high variance.
The spaces spanned by the columns of the matrices are orthogonal to each other, that is $\mathcal{Q} \perp \mathcal{W}$.}
\label{fig:illustrationApproach}
\end{minipage}
\end{figure}

An illustration of this idea is presented in Fig. \ref{fig:illustrationApproach}.
In this figure, the matrices $\Sigma_i$ are seen as points on the SPD matrix manifold $\spd{D}$.
The goal of our method is to find transformations $Q$ and $W$ that map the matrices onto two separate manifolds of lower dimension - the stationary and non-stationary space, respectively.
The matrices in the stationary space will exhibit small variation, while the non-stationarities will be captured in the non-stationary space where the variation of the matrices will be greater.
The transformations $Q$ and $W$ may be chosen to be orthogonal to each other, producing well separated $\mathfrak{s}$- and $\mathfrak{n}$-spaces.
That is, $W \in \mathcal{Q}^\perp$ for $\mathcal{Q} = \spn Q$ (likewise $Q \in \mathcal{W}^\perp$), where $\perp$ denotes the orthogonal complement.

Put formally, our objective is to find a rank-$m$ transformation matrix $Q \in \mathbb{R}^{D \times m}$ which maps $\Sigma_i \in \spd{D}$ to $\widehat{\Sigma}^\st_i \in \spd{m}$ for $m < D$ such that the log-determinant distance between the compressed centered matrices $\widehat{\Sigma}^\st_i = Q^\top \tilde{\Sigma_i} Q$ and their mean, which for the centered matrices is $\mathbb{I}$, is minimized.
Note that the space spanned by the columns of $Q$ is of importance, and not the specific columns themselves.
So, we may optimize $Q$ over the Grassmann manifold \cite{edelman1998geometry}, $\Gr = \left\lbrace \text{span}(Q) \; : \; Q \in \mathbb{R}^{D \times m}, \; Q^\top Q = \mathbb{I} \right\rbrace$, the set of all $m$-dimensional linear subspaces of $\mathbb{R}^{D \times D}$.

In practice, for optimization we employ a Riemannian trust-regions method described in \citet{absil2009optimization} and implemented efficiently in \citet{manopt}.

\subsection{A generic geometric formulation}

Given the strong relation between the log-determinant metric and the AIRM \cite{bhatia2009positive}, a natural progression is to incorporate the AIRM into the cost function.
To understand why it would be beneficial to use the AIRM it is necessary first to briefly discuss the geometry of the SPD matrix manifold.

\begin{figure}[t]
\begin{minipage}[b]{0.35\textwidth}
\centering
  \includegraphics[width=1\textwidth]{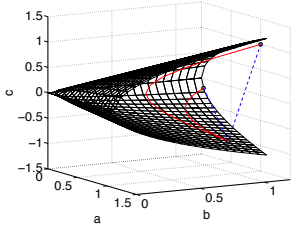}
\end{minipage}
\begin{minipage}[b]{0.60\textwidth}
      \vspace*{10mm}
\caption{Comparison between Euclidean (blue straight dashed lines) and Riemannian (red curved solid lines) distances measured between points of the space $\spd{2}$.}
\label{fig:psd2d}
\vspace*{10mm}
\end{minipage}
\end{figure}

When equipped with the Frobenius inner product $\langle A, B \rangle_{\mathcal{F}} = \tr(A^\top B)$, the set $\spd{n}$ of SPD matrices of size $n \times n$, belongs to a Euclidean space.
In this case, similarity between SPD matrices can be measured simply by using the Euclidean distance derived from the Euclidean norm.
This is readily seen in the following example for $2\times2$ SPD matrices.
A matrix $A \in \spd{2}$ can be written as ${A = \left[ \begin{smallmatrix}
a & c \\
c & b  \end{smallmatrix} \right]}$ with ${ab - c^2 > 0 }$, $a>0$ and $b>0$.
Then matrices in $\spd{2}$ can be represented as points in $\mathbb{R}^3$ and
the constraints can be plotted as a convex cone whose interior is populated by the SPD matrices (see Fig.~\ref{fig:psd2d}). 
In this representation, the Euclidean geometry of symmetric matrices then implies that distances are computed along straight lines.

Despite its simplicity, the Euclidean geometry has several drawbacks and is not always well suited for SPD matrices~\citep{fletcher2004principal, arsigny2007geometric}. 
For example, due to an artifact referred to as the \emph{swelling effect} \cite{arsigny2007geometric}, for a task as simple as averaging two matrices, it may occur that the determinant of the average is larger than any of the two matrices.
Another drawback, illustrated in Fig.~\ref{fig:psd2d} and documented by~\citet{fletcher2004principal}, is the fact that this geometry forms a non-complete space.
Hence, in this Euclidean space interpolation between SPD matrices is possible, but extrapolation may produce indefinite matrices, leading to uninterpretable solutions. 

An efficient alternative which addresses these issues is to consider the space of SPD matrices as a curved space, namely a Riemannian manifold.
Of the possible Riemannian distances, the AIRM, due to its favorable mathematical properties, is widely used in many applications (see, for example \cite{fletcher2004principal, pennec2006riemannian}). 
It is defined for any $X, Y \in \spd{D}$ as:
\begin{equation}
\delta_\mathrm{r}^2\left( X,Y \right) = \norm{ \log \left( X^{-1/2}YX^{-1/2} \right) }^2_\mathbb{F} ,
\label{eq:AIRM}
\end{equation}
where $\log(\cdot )$ is the matrix logarithm and $\norm{X}^2_\mathbb{F} = \tr \left( X^\top X \right)$ is the Frobenius norm.

In the curved space, the geodesics between matrices obtained by the AIRM are computed on curved lines as illustrated in Fig.~\ref{fig:psd2d} for the space $\spd{2}$.
Symmetric matrices with null and infinite eigenvalues (i.e., those which lie on the boundary of the convex cone, but not in it) are both at an infinite distance from any SPD matrix on the manifold (\emph{within} the cone).

So, let us also consider a cost function of the same form defined w.r.t.~the AIRM.\footnote{Other metrics such as the Euclidean metric or the log-Euclidean metric \cite{arsigny2007geometric} may also be used.}
A general expression for our geometry-aware SSA (gaSSA) is then:
\begin{definition}[gaSSA]
\label{def:gaSSA}
gaSSA w.r.t.~a metric $\delta$ is defined as
\begin{equation}
\label{eq:costFunctionGA}
\widehat{Q} = \underset{Q \in \Gr (D,m)}{\mathrm{argmin}} \; \sum_i \delta^2 \left( Q^\top \tilde{\Sigma_i} Q,  \mathbb{I} \right),
\end{equation}
\end{definition}
where $\tilde{\Sigma}_i = Z \Sigma_i Z^\top$ are the matrices whitened with $Z = \bar{\Sigma}^{-1/2}$ and $\bar{\Sigma} = \underset{\Sigma \in \spd{D}}{\mathrm{argmin}} \; \delta^2 \left( \Sigma_i, \Sigma \right)$ is the matrix mean w.r.t.~$\delta$.
We note that this cost function is similar in spirit to the one in \citet{harandi2014manifold} and can be considered an unsupervised version of it.
In the next section we will show that the need for matrix whitening can be alleviated.

The Euclidean gradient w.r.t.~$Q$ of the cost function, used for the optimization, can be found in~\cite{horev2015intrinsic}.

\subsection{Symmetries and invariance properties}
\label{section:invariance}

We now discuss the symmetries of our optimization problem and the invariance properties of our chosen metrics.
These properties will enable us to significantly simplify our problem.
For brevity we will state the results in terms of the AIRM, but the same holds true for the log-determinant metric.

Our key observation stems from the fact that $\delta_\mathrm{r}$ and $\delta_\mathrm{s}$ are invariant to congruent transformations of the form $X \mapsto P^H X P$  for $P \in GL_D(\mathbb{C})$ and $P^H = \overline{P}^\top$ is the conjugate transpose \cite{bhatia2009positive}.
However, since our discussion is limited to real matrices, we have
\begin{equation}
\label{eq:congruenceInvariance}
\delta^2_\mathrm{r} \left( X,Y \right) =  \delta^2_\mathrm{r} \left( P^\top X P , P^\top Y P \right).
\end{equation}
for $X,Y \in \spd{D}$ and a real-valued invertible matrix $P$.
This is a crucial point since the whitening matrix $Z$ and, more importantly, the \emph{mixing matrix} $A$, act on the covariance matrices in this way.

\begin{proposition}
Let $\mathbf{\Lambda} = \left\lbrace \Lambda_i \right\rbrace_{i=1}^n$ for $\Lambda_i \in \spd{n}$ be a set of SPD matrices of size $n \times n$ and let $\Sigma_i = A \Lambda_i A^\top$ for some real-valued invertible matrix $A$. Denoting the Riemannian mean of $\mathbf{\Lambda}$ by $\bar{\Lambda}$, the Riemannian mean $\bar{\Sigma}$ of the set $\mathbf{\Sigma} = \left\lbrace \Sigma_i \right\rbrace_{i=1}^n$ is given by $\bar{\Sigma} = A \bar{\Lambda} A^\top$.
\end{proposition}
\begin{proof}
The Riemannian mean of the set $\mathbf{\Lambda}$ is defined as $\bar{\Lambda} = \underset{\Lambda \in \spd{n}}{\mathrm{argmin}} \; \sum_i \delta^2_\mathrm{r} \left( \Lambda_i, \Lambda \right)$.
Using the congruence invariance (Eq. (\ref{eq:congruenceInvariance})) we have $\delta^2_\mathrm{r} \left( \Lambda_i, \Lambda \right) = \delta^2_\mathrm{r} \left( \Sigma_i, A \Lambda A^\top \right)$ and the result follows.
\end{proof}

Using the above we obtain several useful equivalence relations.
\begin{corollary}
The following expressions are equivalent:
\begin{equation}
\label{eq:distanceEquivalence}
\delta^2_\mathrm{r} \left( \tilde{\Sigma}_i, \mathbb{I} \right) = \delta^2_\mathrm{r} \left( \Sigma_i, \bar{\Sigma} \right) = \delta^2_\mathrm{r} \left( \Gamma_i, \bar{\Gamma} \right),
\end{equation}
where $\Gamma_i$ is the covariance matrix of the \emph{unmixed} sources in the $i$-th epoch.
\end{corollary}

We have essentially shown that both the whitening operation and the mixing matrix $A$ do not affect the distance between the covariance matrices of the original unmixed signals.
We can then re-write our optimization problem as
\begin{align}
\label{eq:gaSSA_whitening}
\nonumber \widehat{Q} & =  \underset{Q \in \Gr (D,m)}{\mathrm{argmin}} \; \sum_i \delta^2 \left( Q^\top \tilde{\Sigma_i} Q,  \mathbb{I} \right) =  \underset{Q \in \Gr (D,m)}{\mathrm{argmin}} \; \sum_i \delta^2 \left( Q^\top A \Gamma_i A^\top Q,  Q^\top A \bar{\Gamma} A^\top Q \right)  \\
& =  \underset{Q' \in \Gr (D,m)}{\mathrm{argmin}} \; \sum_i \delta^2 \left( Q'^\top \Gamma_i Q',  Q'^\top \bar{\Gamma} Q' \right).
\end{align}

One may remark that $A^\top Q$ no longer has orthonormal columns and so does not belong to the Grassmann manifold.
Indeed this is true.
The final transition is due to the observation that the solution to our optimization problem is not unique.
Rather, since we are interested in recovering the stationary subspace and not the exact sources themselves, the solution is invariant to any transformation (e.g., subspace scaling and rotation) acting within each of the $\st$- and $\ns$-spaces separately. 
Furthermore, we have chosen $W \in \mathcal{Q}^\perp$, and so the $\st$-space is orthogonal to the $\ns$-space.
Now, choosing orthogonal bases within each of the subspaces we may restrict ourselves to \emph{orthogonal} mixing matrices $A$ and find a transformation $Q$ which lies in the Grassman manifold.

The final result is quite remarkable.
First, it shows that our problem is essentially agnostic to the mixing matrix.
Secondly, it eliminates the need to pre-whiten the matrices.
This is useful in situations where it is not appropriate to whiten the data in advance, for example, when working with nearly ill-conditioned matrices.
Furthermore, this eliminates any error that may be caused by inaccurate estimation of the matrix mean.

In conclusion, we have two variations of gaSSA given in the first and last terms of Eq. (\ref{eq:gaSSA_whitening}).
The difference between the two is whether or not the input covariance matrices are whitened.
Our analysis shows that whitening does not improve performance, and may in fact lead to a degradation of the results in certain cases.
So, we claim that it is in general preferable not to whiten the matrices.
In terms of the chosen metric, we do not expect a significant difference when using $\delta_\mathrm{r}$ vs.~$\delta_\mathrm{s}$.
In the following section we will present experimental evidence to support these claims.

\section{Experimental results}

In this section we present experimental results on synthetic and data taken from a real BCI experiment.
We compare the performance of gaSSA to the existing SSA and investigate the effects of matrix whitening and the choice of metric.

\subsection{Toy data}

For our first experiment we generated data following the SSA model as a mixture of stationary and non-stationary sources.
To generate non-stationarity in the data we used a slightly modified version of the scheme provided in the SSA toolbox \cite{muller2011stationary} and detailed in its user manual.
Here we bring only a brief description:

The elements of the mixing matrix $A$ are chosen uniformly from the range $\left[ -0.5,0.5 \right]$ and its columns are normalized to $1$.
The distribution of the $\st$-sources is constant over all epochs, namely $s^\st (t) \sim \mathcal{N}\left(0, \Lambda^\st \right)$.
In the SSA toolbox, $\Lambda^\st$ is taken to be the identity matrix, however, we choose $\Lambda^\st$ to be a random matrix of the form $\Lambda^\st = B \Gamma B^\top$ for an orthogonal matrix $B$ and diagonal matrix $\Gamma$.

The $\ns$-sources are correlated with the $\st$-sources, and for the $i$-th epoch $\tau_i = \left[ t_0(i), \hdots, t_0(i) + T \right]$ they are given by $s^\ns(t) = C_i s^\st (t) + Y^\ns (t)$ for $\quad t \in \tau_i$,
where $C_i \in \mathbb{R}^{(D-m)\times m}$ and $Y^\ns (t) \sim \mathcal{N}\left( \mu_i, \Lambda^\ns_i \right)$.
The covariance matrices $\Lambda^\ns_i$ are generated for each epoch in the same way as $\Lambda^\st$.

So, the covariance matrix of the (unmixed) sources in the $i$-th epoch may be written as
\begin{equation}
\Lambda_i = \text{cov}\left( \left[ \begin{matrix} s^\st (t) \\ s^\ns (t) \end{matrix} \right] \right) = \left[ \begin{matrix} \Lambda^\st & \left( C_i \Lambda^\st \right)^\top \\ C_i \Lambda^\st & C_i\Lambda^\st C_i^\top + \Lambda^\ns_i \end{matrix} \right].
\end{equation}

Using the data generated by the scheme above we compared the performance of our method to that of SSA.
We used the AIRM and log-determinant metric both with and without matrix whitening.

As a performance measure we used the distance between the estimated $\ns$-space $\widehat{A}^\ns$ and the true $\ns$-space $A^\ns$.
This is owing to the fact that, as discussed in \citet{von2009stationary}, the $\ns$-space and $\st$-sources are identifiable, while the $\st$-space and $\ns$-sources are not.
To illustrate this, note that to be stationary, the $\st$-sources must consist strictly of stationary sources, while the $\ns$-sources will remain non-stationary even if they include a mixture of stationary signals.
The distance between sub-spaces is computed using $\delta_{\Gr}$, the metric on the Grassmann manifold \cite{absil2004riemannian}.
Shortly, this metric is based on the principal angles between the two spaces.

We generated $50$ epochs of length $T=250$ for several values of $D$ and $m$.
For each pair $(D,m)$ we conducted the experiment $25$ times.
At each iteration the optimization procedure was restarted $5$ times with different initial guesses and the transformation matrix which obtained the lowest cost was selected.

Due to length restrictions we describe only the results for $D=19$ and $m=12$.
Other choices of parameters exhibited the same behavior.
The $\ns$-space errors were identical for $\delta_\mathrm{r}$ and $\delta_\mathrm{s}$ ($0.0067 \pm 0.0001$).
They were consistently lower than SSA ($0.0115 \pm 0.0004$).
For both metrics the scheme without matrix whitening performed slightly better than those including whitening ($0.0063 \pm 0.0001$ vs.~$0.0067 \pm 0.0001$).
These results are consistent with the analysis of Section \ref{section:invariance}.

\subsection{Brain-computer interface}

Next, we applied our method to data taken from the BCI competition IV dataset II.
This dataset contains motor imagery (MI) EEG signals affected by eye movement artifacts.
It was collected in a multi-class setting, with the subjects performing more than 2 different MI tasks.
However, as in \citet{lotte2011regularizing}, we evaluate our algorithms on two-class problems by selecting only signals of left- and right-hand MI trials.

We applied the same pre-processing as described in~\citet{lotte2011regularizing}.
EEG signals were band-pass filtered in $8-30$ Hz, using a $5^\text{th}$ order Butterworth filter.
For each trial, we extracted features from the time segment located from 0.5s to 2.5s after the cue instructing the subject to perform MI.

The data was initially divided into two parts: a training data set and a test data set.
Similarly to \citet{von2009stationary} the first $20 \% $ of the test trials were set aside for adaptation.
The aim of the adaptation part is to mitigate any non-stationarities between the test and the training session.
We then learned the $\st$-space in an \emph{unsupervised} manner over the training and adaptation part.
As before, our method was reinitialized $5$ times and the transformation attaining the lowest cost was chosen.

The performance was measured by means of the classification rate on the test set.
We used the following naive classifier, referred to as \emph{minimum distance to the mean} (MDM) in \citet{barachant2012multiclass}:
Using the labels of the training set, we compute the mean (in the $\st$-space) for each of the two classes.
Then, we classify the compressed covariance matrices in the test set according to their distance to the class means; each matrix is assigned the class to which it is closer.

The original data is comprised of $22$ signals.
Since the true number of stationary signals is unknown, we repeated the experiment for several values in the range $m \in [10,18]$.
Due to lack of space, we bring the results only for the values $m=10$ and $m=14$, the former being an example of challenging task and the latter an example of an intermediate one.
The results for the nine subjects in the dataset are summarized in Table~\ref{tab:BCI_results10}.
The full results can be found in the supplementary material.

The results show that our method outperforms SSA for most subjects.
As predicted, the methods without pre-whitening generally performed better than those which included whitening of the covariance matrices.
In this more complex setting, more accurate estimation of the mixing matrix does not guarantee better classification.
In terms of the metric, we see that $\delta_\mathrm{r}$ and $\delta_\mathrm{s}$ perform roughly the same.

\begin{table}[t]
\centering
\small
\caption{Classification accuracy for $m=10$ (top) and $m=14$ (bottom) $\st$-sources. Best results are highlighted in boldface. (w) and (nw) signify that matrix whitening was / was not performed.}
\label{tab:BCI_results10}
\begin{tabular}{*{11}{@{\hspace{1.5pt}}c@{\hspace{1.5pt}}}}
\toprule
subject $\#$ & 1 & 2 & 3 & 4 & 5 & 6 & 7 & 8 & 9 & avg \\
\midrule
gaSSA $\delta_\mathrm{r} $ (w) & 48.96 & 55.65 & 63.48 & 58.22 & \textbf{60.69} & 56.26 & 68.09 & 71.30 & 48.70 & 59.04 \\
gaSSA $\delta_\mathrm{r} $ (nw) & 73.91 & 59.13 & \textbf{91.3} & \textbf{69.3} & 57.67 & \textbf{66.43} & 58.70 & \textbf{93.91} & \textbf{86.09} & \textbf{72.94} \\
gaSSA $\delta_\mathrm{s} $ (w) & 49.83 & 55.65 & 63.57 & 61.43 & 58.9 & 55.7 & \textbf{68.17} & 72.17 & 47.34 & 59.2 \\
gaSSA $\delta_\mathrm{s} $ (nw) & \textbf{74.13} & \textbf{60} & \textbf{91.3} & 68.52 & 57.03 & 63.96 & 59.48 & \textbf{93.91} & 85.99  & 72.7 \\
SSA & 47.74 & 57.16 & 60 & 57.61 & 56.02 & 54.17 & 67.3 & 75.89 & 53.82 & 58.86 \\
\midrule
gaSSA $\delta_\mathrm{r} $ (w) & 65.02 & 53.91 & 72.35 & 65.22 & \textbf{62.61} & 59.13 & 72.93 & 77.39 & 77.39 & 67.33 \\
gaSSA $\delta_\mathrm{r} $ (nw) & \textbf{74.4} & \textbf{58.26} & 88.74 & 67.83 & 60 & \textbf{65.22} & 68.91 & \textbf{94.78} & \textbf{87.83} & \textbf{74} \\
gaSSA $\delta_\mathrm{s} $ (w) & 65.02 & 53.91 & 72.35 & 65.22 & \textbf{62.61} & 58.87 & 72.83 & 77.39 & 76.52 & 67.19 \\
gaSSA $\delta_\mathrm{s} $ (nw) & \textbf{74.4} & \textbf{58.26} & \textbf{88.78} & 67.83 & 60 & 63.83 & 68.15 & 93.91 & \textbf{87.83} & 73.67 \\
SSA & 67.05 & 54.72 & 66.7 & \textbf{68.7} & 56.87 & 55.57 & \textbf{73.48} & 75.65 & 77.39 & 66.24 \\
\bottomrule
\end{tabular}
\end{table}

\section{Conclusion}
\label{section:discussion}

We presented a covariance-based method for unsupervised stationary subspace analysis.
The problem was phrased in terms of the distance between matrices and not, as in SSA, using the divergence between probability distributions.
Owing to the symmetries of the problem and the invariance properties of the geometries, we derived useful equivalence relations.
Experiments on both synthetic and BCI data supported our theoretical analysis and showed that our method outperforms SSA.



\small

\bibliography{citations}
\bibliographystyle{abbrvnat}

\end{document}